%% file: Paper.tex
\documentclass[letterpaper]{article}
\usepackage{algorithm}
\usepackage{algorithmic}
\usepackage{amsmath}
\usepackage{amssymb}
\usepackage{amsthm}
\usepackage{bbm}
\usepackage{bm}
\usepackage{color}
\usepackage{dsfont}
\usepackage{graphicx}
\usepackage{proceed2e}
\usepackage{subfigure}
\usepackage{times}
\usepackage{url}

\newcommand{\commentout}[1]{}
\newcommand{\junk}[1]{}

\newcommand{\etal}{\emph{et al.}}

\newcommand{\omm}{{\tt OMM}}

\newtheorem{theorem}{Theorem}

\newtheorem{lemma}{Lemma}

\newcommand{\bu}{{\bf u}}

\newcommand{\bw}{{\bf w}}

\newcommand{\cI}{\mathcal{I}}

\newcommand{\cO}{\mathcal{O}}

\newcommand{\eps}{\varepsilon}

\newcommand{\realset}{\mathbb{R}}

\newcommand{\abs}[1]{\left|#1\right|}

\newcommand{\E}[2]{\mathbb{E}_{#2} \! \left[#1\right]}

\newcommand{\floors}[1]{\left\lfloor#1\right\rfloor}
\newcommand{\I}[1]{\mathds{1} \! \left\{#1\right\}}

\newcommand{\set}[1]{\left\{#1\right\}}

\DeclareMathOperator*{\argmax}{arg\,max\,}

\begin{document}

\title{Matroid Bandits: Fast Combinatorial Optimization with Learning}

\author{Branislav Kveton\textnormal{, }Zheng Wen\textnormal{, }Azin Ashkan\textnormal{, }Hoda Eydgahi\textnormal{, and }Brian Eriksson \\
Technicolor Labs \\
Los Altos, CA \\
\emph{\{branislav.kveton,zheng.wen,azin.ashkan,hoda.eydgahi,brian.eriksson\}@technicolor.com}}

\maketitle

\begin{abstract}
A matroid is a notion of independence in combinatorial optimization which is closely related to computational efficiency. In particular, it is well known that the maximum of a constrained modular function can be found greedily if and only if the constraints are associated with a matroid. In this paper, we bring together the ideas of bandits and matroids, and propose a new class of combinatorial bandits, \emph{matroid bandits}. The objective in these problems is to learn how to maximize a modular function on a matroid. This function is stochastic and initially unknown. We propose a practical algorithm for solving our problem, \emph{Optimistic Matroid Maximization ($\omm$)}; and prove two upper bounds, gap-dependent and gap-free, on its regret. Both bounds are sublinear in time and at most linear in all other quantities of interest. The gap-dependent upper bound is tight and we prove a matching lower bound on a partition matroid bandit. Finally, we evaluate our method on three real-world problems and show that it is practical.
\end{abstract}

\input{Introduction.tex}

\input{Matroids.tex}

\input{Bandits.tex}

\input{Model.tex}

\input{Algorithm.tex}

\input{Analysis.tex}

\input{Experiments.tex}

\input{Experiment1.tex}

\input{Experiment2.tex}

\input{Experiment3.tex}

\input{RelatedWork.tex}

\input{Conclusions.tex}

\bibliographystyle{plain}
\bibliography{References}

\input{Appendix.tex}

\end{document}

%% file: Introduction.tex
%!TEX root = Paper.tex

\section{Introduction}
\label{sec:introduction}

Combinatorial optimization is a well-established field that has many practical applications, ranging from resource allocation \cite{katoh01combinatorial} to designing network routing protocols \cite{oliveira05survey}. Modern combinatorial optimization problems are often so massive that even low-order polynomial-time solutions are not practical. Fortunately, many important problems, such as finding a minimum spanning tree, can be solved greedily. Such problems can be often viewed as optimization on a \emph{matroid} \cite{whitney35abstract}, a notion of independence in combinatorial optimization which is closely related to computational efficiency. In particular, it is well known that the maximum of a constrained modular function can be found \mbox{greedily if} and only if all feasible solutions to the problem are the independent sets of a matroid \cite{edmonds71matroids}. Matroids are common in practice because they generalize many notions of independence, such as linear independence and forests in graphs.

In this paper, we propose an algorithm for learning how to maximize a stochastic modular function on a matroid. The modular function is represented as the sum of the weights of up to $K$ items, which are chosen from the ground set $E$ of a matroid, which has $L$ items. The weights of the items are stochastic and represented as a vector $\bw \in [0, 1]^L$. The vector $\bw$ is drawn i.i.d. from a probability distribution $P$. The distribution $P$ is initially unknown and we learn it by interacting repeatedly with the environment.

Many real-world optimization problems can be formulated in our setting, such as building a spanning tree for network routing \cite{oliveira05survey}. When the delays on the links of the network are stochastic and their distribution is known, this problem can be solved by finding a minimum spanning tree. When the distribution is unknown, it must be learned, perhaps by exploring routing networks that seem initially suboptimal. We return to this problem in our experiments.

This paper makes three main contributions. First, we bring together the concepts of matroids \cite{whitney35abstract} and bandits \cite{lai85asymptotically,auer02finitetime}, and propose a new class of combinatorial bandits, \emph{matroid bandits}. On one hand, matroid bandits can be viewed as a new learning framework for a broad and important class of combinatorial optimization problems. On the other hand, matroid bandits are a class of $K$-step bandit problems that can be solved both computationally and sample efficiently.

Second, we propose a simple greedy algorithm for solving our problem, which explores based on the optimism in the face of uncertainty. We refer to our approach as \emph{Optimistic Matroid Maximization ($\omm$)}. $\omm$ is both computationally and sample efficient. In particular, the time complexity of $\omm$ is $O(L \log L)$ per episode, comparable to that of sorting $L$ numbers. Moreover, the expected cumulative regret of $\omm$ is sublinear in the number of episodes, and at most linear in the number of items $L$ and the maximum number of chosen items $K$.

Finally, we evaluate $\omm$ on three real-world problems. In the first problem, we learn routing networks. In the second problem, we learn a policy for assigning loans in a microfinance network that maximizes chances that the loans are repaid. In the third problem, we learn a movie recommendation policy. All three problems can be solved efficiently in our framework. This demonstrates that $\omm$ is practical and can solve a wide range of real-world problems.

We adopt the following notation. We write $A + e$ instead of $A \cup \set{e}$, and $A + B$ instead of $A \cup B$. We also write $A - e$ instead of $A \setminus \set{e}$, and $A - B$ instead of $A \setminus B$.

%% file: Matroids.tex
%!TEX root = Paper.tex

\section{Matroids}
\label{sec:matroids}

A \emph{matroid} is a pair $M = (E, \cI)$, where $E = \set{1, \dots, L}$ is a set of $L$ items, called the \emph{ground set}, and $\cI$ is a family of subsets of $E$, called the \emph{independent sets}. The family $\cI$ is defined by the following properties. First, $\emptyset \in \cI$. Second, every subset of an independent set is independent. Finally, for any $X \in \cI$ and $Y \in \cI$ such that $\abs{X} = \abs{Y} + 1$, there must exist an item $e \in X - Y$ such that $Y + e \in \cI$. This \mbox{is known} as the \emph{augmentation property}. We denote by:
\begin{align}
  E(X) = \set{e: e \in E - X, \ X + e \in \cI}
\end{align}
the set of items that can be added to set $X$ such that the set remains independent.

A set is a \emph{basis} of a matroid if it is a maximal independent set. All bases of a matroid have the same cardinality \cite{whitney35abstract}, which is known as the \emph{rank} of a matroid. In this work, we denote the rank by $K$.

A \emph{weighted matroid} is a matroid associated with a vector of non-negative weights $\bw \in (\realset^+)^L$. The $e$-th entry of $\bw$, $\bw(e)$, is the weight of item $e$. We denote by:
\begin{align}
  f(A, \bw)=\sum_{e \in A} \bw(e)
  \label{eq:return}
\end{align}
the sum of the weights of all items in set $A$. The problem of finding a \emph{maximum-weight basis} of a matroid:
\begin{align}
  A^\ast =
  \argmax_{A \in \cI} f(A, \bw) =
  \argmax_{A \in \cI} \sum_{e \in A} \bw(e)
  \label{eq:optimal basis}
\end{align}
is a well-known combinatorial optimization problem. This problem can be solved greedily (Algorithm~\ref{alg:greedy}). The greedy algorithm has two main stages. First, $A^\ast$ is initialized to $\emptyset$. Second, all items in the ground set are sorted according to their weights, from the highest to the lowest, and greedily added to $A^\ast$ in this order. The item is added to the set $A^\ast$ only if it does not make the set dependent.

\begin{algorithm}[t]
  \caption{The greedy method for finding a maximum-weight basis of a matroid \cite{edmonds71matroids}.}
  \label{alg:greedy}
  \begin{algorithmic}
    \STATE {\bf Input:} Matroid $M = (E, \cI)$, weights $\bw$
    \STATE
    \STATE $A^\ast \gets \emptyset$
    \STATE Let $e_1, \dots, e_L$ be an ordering of items such that:
    \STATE \quad $\bw(e_1) \geq \ldots \geq \bw (e_L)$
    \FORALL{$i = 1, \dots, L$}
      \IF{$(e_i \in E(A^\ast))$}
        \STATE $A^\ast \gets A^\ast + e_i$
      \ENDIF
    \ENDFOR
  \end{algorithmic}
\end{algorithm}

%% file: Bandits.tex
%!TEX root = Paper.tex

\section{Matroid Bandits}
\label{sec:matroid bandits}

A \emph{minimum spanning tree} is a maximum-weight basis of a matroid. The ground set $E$ of this matroid are the edges of a graph. A set of edges is considered to be independent if it does not contain a cycle. Each edge $e$ is associated with a weight $\bw(e) = u_{\max} - \bu(e)$, where $u_{\max} = \max_e \bu(e)$ and $\bu(e)$ is the weight of edge $e$ in the original graph.

The minimum spanning tree cannot be computed when the weights $\bw(e)$ of the edges are unknown. This may happen in practice. For instance, consider the problem of building a routing network, which is represented as a spanning tree, where the expected delays on the links of the \mbox{network are} initially unknown. In this work, we study a variant of maximizing a modular function on a matroid that can address this kind of problems.

%% file: Model.tex
%!TEX root = Paper.tex

\subsection{Model}
\label{sec:model}

We formalize our learning problem as a matroid bandit. A \emph{matroid bandit} is a pair $(M, P)$, where $M$ is a matroid and $P$ is a probability distribution over the weights $\bw \in \realset^L$ of items $E$ in $M$. The $e$-th entry of $\bw$, $\bw(e)$, is the weight of item $e$. The weights $\bw$ are stochastic and drawn i.i.d. from the distribution $P$. We denote the expected weights of the items by $\bar{\bw} = \mathbb{E}[\bw]$ and assume that each of these weights is non-negative, $\bar{\bw}(e) \geq 0$ for all $e \in E$.

Each item $e$ is associated with an \emph{arm} and we assume that \emph{multiple arms} can be pulled. A subset of arms $A \subseteq E$ can be pulled if and only if $A$ is an independent set. The return for pulling arms $A$ is $f(A, \bw)$ (Equation~\ref{eq:return}), the sum of the weights of all items in $A$. After the arms $A$ are pulled, we observe the weight of each item in $A$, $\bw(e)$ for all $e \in A$. This model of feedback is known as \emph{semi-bandit} \cite{audibert14regret}.

We assume that the matroid $M$ is known and that the distribution $P$ is unknown. Without loss of generality, we assume that the support of $P$ is a bounded subset of $[0, 1]^L$. We would like to stress that we do not make any structural assumptions on $P$.

The optimal solution to our problem is a maximum-weight basis in expectation:
\begin{align}
  A^\ast =
  \argmax_{A \in \cI} \E{f(A, \bw)}{\bw} =
  \argmax_{A \in \cI} \sum_{e \in A} \bar{\bw}(e).
  \label{eq:optimal arm}
\end{align}
The above optimization problem is equivalent to the problem in Equation~\ref{eq:optimal basis}. Therefore, it can be solved greedily by Algorithm~\ref{alg:greedy} when the expected weights $\bar{\bw}$ are known.

Our learning problem is \emph{episodic}. In episode $t$, we choose a basis $A^t$ and gain $f(A^t, \bw_t)$, where $\bw_t$ is the realization of the stochastic weights in episode $t$. Our goal is to learn a policy, a sequence of bases, that minimizes the \emph{expected cumulative regret} in $n$ episodes:
\begin{align}
  R(n) = \E{\sum_{t = 1}^n R_t(\bw_t)}{\bw_1, \dots, \bw_n},
  \label{eq:cumulative regret}
\end{align}
where $R_t(\bw_t) = f(A^\ast, \bw_t) - f(A^t, \bw_t)$ is the regret in episode $t$.

%% file: Algorithm.tex
%!TEX root = Paper.tex

\subsection{Algorithm}
\label{sec:algorithm}

\begin{algorithm}[t]
  \caption{$\omm$: Optimistic matroid maximization.}
  \label{alg:bandit}
  \begin{algorithmic}
    \STATE {\bf Input:} Matroid $M = (E, \cI)$
    \STATE
    \STATE // Initialization
    \STATE Observe $\bw_0 \sim P$
    \STATE $\hat{w}_{e, 1} \gets \bw_0(e) \hspace{1.81in} \forall e \in E$
    \STATE $T_e(0) \gets 1 \hspace{2.02in} \forall e \in E$
    \STATE
    \FORALL{$t = 1, \dots, n$}
      \STATE // Compute UCBs
      \STATE $U_t(e) \gets \hat{w}_{e, T_e(t - 1)} + c_{t - 1, T_e(t - 1)} \hspace{0.53in} \forall e \in E$
      \STATE
      \STATE // Find a maximum-weight basis with respect to $U_t$
      \STATE $A^t \gets \emptyset$
      \STATE Let $e^t_1, \dots, e^t_L$ be an ordering of items such that:
      \STATE \quad $U_t(e^t_1) \geq \ldots \geq U_t(e^t_L)$
      \FORALL{$i = 1, \dots, L$}
        \IF{$(e^t_i \in E(A^t))$}
          \STATE $A^t \gets A^t + e^t_i$
        \ENDIF
      \ENDFOR
      \STATE Observe $\set{\bw_t(e): e \in A^t}$, where $\bw_t \sim P$
      \STATE
      \STATE // Update statistics
      \STATE $T_e(t) \gets T_e (t-1) \hspace{1.43in} \forall e \in E$
      \STATE $T_e(t) \gets T_e(t) + 1 \hspace{1.43in} \forall e \in {A^t}$
      \STATE $\displaystyle \hat{w}_{e, T_e(t)} \gets
      \frac{T_e(t - 1) \hat{w}_{e, T_e(t - 1)} + \bw_t(e)}{T_e(t)} \quad \ \forall e \in {A^t}$
    \ENDFOR
  \end{algorithmic}
\end{algorithm}

Our solution is designed based on the \emph{optimism in the face of uncertainty} principle \cite{munos12optimistic}. In particular, it is a variant of the greedy method for finding a maximum-weight basis of a matroid where the expected weight $\bar{\bw}(e)$ of each item $e$ is substituted with its optimistic estimate $U_t(e)$. Therefore, we refer to our approach as \emph{Optimistic Matroid Maximization ($\omm$)}.

The pseudocode of our algorithm is given in Algorithm~\ref{alg:bandit}. The algorithm can be summarized as follows. First, at the beginning of each episode $t$, we compute the \emph{upper confidence bound (UCB)} on the weight of each item $e$:
\begin{align}
  U_t(e) = \hat{w}_{e, T_e(t - 1)} + c_{t - 1, T_e(t - 1)},
  \label{eq:UCB}
\end{align}
where $\hat{w}_{e, T_e(t - 1)}$ is our estimate of $\bar{\bw}(e)$ at the beginning of episode $t$, $c_{t - 1, T_e(t - 1)}$ represents the radius of the confidence interval around this estimate, and $T_e(t - 1)$ is the number of times that $\omm$ chooses item $e$ before episode $t$. Second, we order all items $e$ by their UCBs (Equation~\ref{eq:UCB}), from the highest to the lowest, and then add them greedily to $A^t$ in this order. The item is added to the set $A^t$ only if it does not make the set dependent. Finally, we choose the basis $A^t$, observe the weights of all items in the basis, and update our model $\hat{w}$ of the world.

The radius:
\begin{align}
  c_{t, s} = \sqrt{2 \log(t) / s}
  \label{eq:confidence radius}
\end{align}
is defined such that each upper confidence bound $U_t(e)$ is with high probability an upper bound on the weight $\bar{\bw}(e)$. The role of the UCBs is to encourage exploration of items that are not chosen very often. As the number of episodes increases, the estimates of the weights $\bar{\bw}$ improve and $\omm$ starts exploiting best items. The $\log(t)$ term increases with time $t$ and enforces exploration, to avoid linear regret.

$\omm$ is a greedy algorithm and therefore is extremely computationally efficient. In particular, let the time complexity of checking for independence, $e^t_i \in E(A^t)$, be $O(g(|A^t|))$. Then the time complexity of $\omm$ is $O(L (\log L + g(K)))$ per episode, comparable to that of sorting $L$ numbers. The design of our algorithm is not surprising and is motivated by prior work \cite{gai12combinatorial}. The main challenge is to derive a tight upper bound on the regret of $\omm$, which would reflect the structure of our problem.

%% file: Analysis.tex
%!TEX root = Paper.tex

\section{Analysis}
\label{sec:analysis}

In this section, we analyze the regret of $\omm$. Our analysis is organized as follows. First, we introduce basic concepts and notation. Second, we show how to decompose the regret of $\omm$ in a single episode. In particular, we partition the regret of a suboptimal basis into the sum of the regrets of individual items. This part of the analysis relies heavily on the structure of a matroid and is the most novel. Third, we derive two upper bounds, gap-dependent and gap-free, on the regret of $\omm$. Fourth, we prove a lower bound that matches the gap-dependent upper bound. Finally, we summarize the results of our analysis.

\subsection{Notation}
\label{sec:notation}

Before we present our results, we introduce notation used in our analysis. The \emph{optimal basis} is $A^\ast = \set{a^\ast_1, \dots, a^\ast_K}$. We assume that the items in $A^\ast$ are ordered such that $a^\ast_k$ is the $k$-th item with the highest expected weight. In episode $t$, $\omm$ chooses a basis $A^t = \set{a^t_1, \dots, a^t_K}$, where $a^t_k$ is the $k$-th item chosen by $\omm$. We say that item $e$ is \emph{suboptimal} if it belongs to $\bar{A}^\ast = E - A^\ast$, the \emph{set of suboptimal items}. For any pair of suboptimal and optimal items, $e \in \bar{A}^\ast$ and $a^\ast_k$, we define a \emph{gap}:
\begin{align}
  \Delta_{e, k} = \bar{\bw}(a^\ast_k) - \bar{\bw}(e)
  \label{eq:gap}
\end{align}
and use it as a measure of how difficult it is to discriminate the items. For every item $e \in \bar{A}^\ast$, we define a set:
\begin{align}
  \cO_e = \set{k: \Delta_{e, k} > 0},
  \label{eq:positive gap}
\end{align}
the indices of items in $A^\ast$ whose expected weight is higher than that of item $e$. The cardinality of $\cO_e$ is $K_e = \abs{\cO_e}$.

\subsection{Regret Decomposition}
\label{sec:regret decomposition}

Our decomposition is motivated by the observation that all bases of a matroid are of the same cardinality. As a result, the difference in the expected values of any two bases can be always written as the sum of differences in the weights of their items. In particular:
\begin{align}
  \E{f(A^\ast, \bw) - f(A^t, \bw)}{\bw} =
  \sum_{k = 1}^K \Delta_{a^t_k, \pi(k)},
  \label{eq:pairing}
\end{align}
where $\pi: \set{1, \dots, K} \to \set{1, \dots, K}$ is an arbitrary bijection from $A^t$ to $A^\ast$ such that $\pi(k)$ is the index of the item in $A^\ast$ that is paired with the $k$-th item in $A^t$. In this work, we focus on one particular bijection.

\begin{lemma}
\label{lem:bijection} For any two matroid bases $A^\ast$ and $A^t$, there exists a bijection $\pi: \set{1, \dots, K} \to \set{1, \dots, K}$ such that:
\begin{align*}
  \set{a^t_1, \dots, a^t_{k - 1}, a^\ast_{\pi(k)}} \in \cI
  \quad \forall k = 1, \dots, K.
\end{align*}
In addition, $\pi(k) = i$ when $a^t_k = a^\ast_i$ for some $i$.
\end{lemma}
\begin{proof}
The lemma is proved in Appendix.
\end{proof}

The bijection $\pi$ in Lemma~\ref{lem:bijection} has two important properties. First, $\set{a^t_1, \dots, a^t_{k - 1}, a^\ast_{\pi(k)}} \in \cI$ for all $k$. In other words, $\omm$ can choose item $a^\ast_{\pi(k)}$ at step $k$. However, $\omm$ selects item $a^t_k$. By the design of $\omm$, this can happen only when the UCB of item $a^t_k$ is larger or equal to that of item $a^\ast_{\pi(k)}$. As a result, we know that $U_t(a^t_k) \geq U_t(a^\ast_{\pi(k)})$ in all steps $k$. Second, Lemma~\ref{lem:bijection} guarantees that every optimal item in $A^t$ is paired with the same item in $A^\ast$.

In the rest of the paper, we represent the bijection $\pi$ using an indicator function. The indicator function:
\begin{align}
  \mathds{1}_{e, k}(t) = \I{\exists i: a^t_i = e, \ \pi(i) = k}
  \label{eq:event}
\end{align}
indicates the event that item $e$ is chosen instead of item $a^\ast_k$ in episode $t$. Based on our new representation, we rewrite Equation~\ref{eq:pairing} as:
\begin{align}
  \sum_{k = 1}^K \Delta_{a^t_k, \pi(k)}
  & = \sum_{e \in \bar{A}^\ast} \sum_{k = 1}^K \Delta_{e, k} \mathds{1}_{e, k}(t)
  \nonumber \\
  & \leq \sum_{e \in \bar{A}^\ast} \sum_{k = 1}^{K_e} \Delta_{e, k} \mathds{1}_{e, k}(t)
  \label{eq:episodic regret}
\end{align}
and then bound it from above. The last inequality is due to neglecting the negative gaps.

The above analysis applies to any basis $A^t$ in any episode $t$. The results of our analysis are summarized below.

\begin{theorem}
\label{thm:regret decomposition} The expected regret of choosing any basis $A^t$ in episode $t$ is bounded as:
\begin{align*}
  \E{f(A^\ast, \bw) - f(A^t, \bw)}{\bw} \leq
  \sum_{e \in \bar{A}^\ast} \sum_{k = 1}^{K_e} \Delta_{e, k} \mathds{1}_{e, k}(t).
\end{align*}
The indicator function $\mathds{1}_{e, k}(t)$ indicates the event that item $e$ is chosen instead of item $a^\ast_k$ in episode $t$. When the event $\mathds{1}_{e, k}(t)$ happens, $U_t(e) \geq U_t(a^\ast_k)$. Moreover:
\begin{alignat*}{2}
  \sum_{e \in \bar{A}^\ast} \sum_{k = 1}^{K_e} \mathds{1}_{e, k}(t) & \leq K &
  \qquad & \forall t \\
  \sum_{k = 1}^{K_e} \mathds{1}_{e, k}(t) & \leq 1 &
  \qquad & \forall t, e \in \bar{A}^\ast.
\end{alignat*}
\end{theorem}

The last two inequalities follow from the fact that $\mathds{1}_{e, k}(t)$ is a bijection from $A^t$ to $A^\ast$, every item in the suboptimal basis $A^t$ is matched with one unique item in $A^\ast$.

One remarkable aspect of our regret decomposition is that the exact form of the bijection is not required in the rest of our analysis. We only rely on the properties of $\mathds{1}_{e, k}(t)$ that are stated in Theorem~\ref{thm:regret decomposition}.

\subsection{Upper Bounds}
\label{sec:upper bounds}

Our first result is a gap-dependent bound.

\begin{theorem}[gap-dependent bound]
\label{thm:gap-dependent} The expected cumulative regret of $\omm$ is bounded as:
\begin{align*}
  R(n) \leq
  \sum_{e \in \bar{A}^\ast} \frac{16}{\Delta_{e, K_e}} \log n +
  \sum_{e \in \bar{A}^\ast} \sum_{k = 1}^{K_e} \Delta_{e, k} \frac{4}{3} \pi^2.
\end{align*}
\end{theorem}
\begin{proof}
First, we bound the expected regret in episode $t$ using Theorem~\ref{thm:regret decomposition}:
\begin{align}
  R(n)
  & = \sum_{t = 1}^n \E{\E{R_t(\bw_t)}{\bw_t}}{\bw_1, \dots, \bw_{t - 1}} \nonumber \\
  & \leq \sum_{t = 1}^n \E{\sum_{e \in \bar{A}^\ast} \sum_{k = 1}^{K_e} \Delta_{e, k}
  \mathds{1}_{e, k}(t)}{\bw_1, \dots, \bw_{t - 1}} \nonumber \\
  & = \sum_{e \in \bar{A}^\ast} \sum_{k = 1}^{K_e} \Delta_{e, k}
  \E{\sum_{t = 1}^n \mathds{1}_{e, k}(t)}{\bw_1, \dots, \bw_n}.
  \label{eq:decomposed regret}
\end{align}
Second, we bound the expected cumulative regret associated with each item $e \in \bar{A}^\ast$. The key idea of this step is to decompose the indicator $\mathds{1}_{e, k}(t)$ as:
\begin{align}
  \mathds{1}_{e, k}(t)
  = & \ \mathds{1}_{e, k}(t) \I{T_e(t - 1) \leq \ell_{e, k}} + {} \\
  & \ \mathds{1}_{e, k}(t) \I{T_e(t - 1) > \ell_{e, k}} \nonumber
\end{align}
and choose $\ell_{e, k}$ appropriately. By Lemma~\ref{lem:pulls}, the regret associated with $T_e(t - 1) > \ell_{e, k}$ is bounded as:
\begin{align}
  & \sum_{k = 1}^{K_e} \Delta_{e, k}
  \E{\sum_{t = 1}^n \mathds{1}_{e, k}(t) \I{T_e(t - 1) > \ell_{e, k}}}{\bw_1, \dots, \bw_n} \nonumber \\
  & \quad \leq \sum_{k = 1}^{K_e} \Delta_{e, k} \frac{4}{3} \pi^2
\end{align}
when $\ell_{e, k} = \floors{\frac{8}{\Delta_{e, k}^2} \log n}$. For the same value of $\ell_{e, k}$, the regret associated with $T_e(t - 1) \leq \ell_{e, k}$ is bounded as:
\begin{align}
  & \sum_{k = 1}^{K_e} \Delta_{e, k}
  \E{\sum_{t = 1}^n \mathds{1}_{e, k}(t) \I{T_e(t - 1) \leq \ell_{e, k}}}
  {\bw_1, \dots, \bw_n} \nonumber \\
  & \quad \leq \max_{\bw_1, \dots, \bw_n} \Bigg[
  \sum_{t = 1}^n \sum_{k = 1}^{K_e} \Delta_{e, k} \mathds{1}_{e, k}(t) \times {}
  \label{eq:trivial regret} \\
  & \hspace{0.91in} \I{T_e(t - 1) \leq \frac{8}{\Delta_{e, k}^2} \log n}\Bigg]. \nonumber
\end{align}
The next step of our proof is based on three observations. First, the gaps are ordered such that $\Delta_{e, 1} \geq \ldots \geq \Delta_{e, K_e}$. Second, by the design of $\omm$, the counter $T_e(t)$ increases when the event $\mathds{1}_{e, k}(t)$ happens, for any $k$. Finally, by Theorem~\ref{thm:regret decomposition}, $\sum_{k = 1}^{K_e} \mathds{1}_{e, k}(t) \leq 1$ for any given $e$ and $t$. Based on these facts, we bound Equation~\ref{eq:trivial regret} from above by:
\begin{align}
  \!\!\!\!\!\!\!
  \left[\Delta_{e, 1} \frac{1}{\Delta_{e, 1}^2} + \sum_{k = 2}^{K_e} \Delta_{e, k} \!
  \left(\frac{1}{\Delta_{e, k}^2} - \frac{1}{\Delta_{e, k - 1}^2}\right)\!\right] \! 8 \log n.
\end{align}
By Lemma~\ref{lem:multiple optimal pulls}, the above term is bounded by $\displaystyle \frac{16}{\Delta_{e, K_e}} \log n$. Finally, we combine all of the above inequalities and get:
\begin{align}
  & \sum_{k = 1}^{K_e} \Delta_{e, k}
  \E{\sum_{t = 1}^n \mathds{1}_{e, k}(t)}{\bw_1, \dots, \bw_n} \nonumber \\
  & \quad \leq \frac{16}{\Delta_{e, K_e}} \log n + \sum_{k = 1}^{K_e} \Delta_{e, k} \frac{4}{3} \pi^2.
  \label{eq:per-item regret}
\end{align}
Our main claim is obtained by summing over all suboptimal items $e \in \bar{A}^\ast$.
\end{proof}

We also prove a gap-free bound.

\begin{theorem}[gap-free bound]
\label{thm:gap-free} The expected cumulative regret of $\omm$ is bounded as:
\begin{align*}
  R(n) \leq 8 \sqrt{K L n \log n} + \frac{4}{3} \pi^2 K L.
\end{align*}
\end{theorem}
\begin{proof}
The key idea is to decompose the expected cumulative regret of $\omm$ into two parts, where the gaps are larger than $\eps$ and at most $\eps$. We analyze each part separately and then set $\eps$ to get the desired result.

Let $K_{e, \eps}$ be the number of optimal items whose expected weight is higher than that of item $e$ by more than $\eps$ and:
\begin{align}
  Z_{e, k}(n) = \E{\sum_{t = 1}^n \mathds{1}_{e, k}(t)}{\bw_1, \dots, \bw_n}.
\end{align}
Then, based on Equation~\ref{eq:decomposed regret}, the regret of $\omm$ is bounded for any $\eps$ as:
\begin{align}
  R(n)
  \leq & \sum_{e \in \bar{A}^\ast} \sum_{k = 1}^{K_{e, \eps}} \Delta_{e, k} Z_{e, k}(n) + {}
  \label{eq:gap-free split} \\
  & \sum_{e \in \bar{A}^\ast} \sum_{k = K_{e, \eps} + 1}^{K_e} \Delta_{e, k} Z_{e, k}(n).
  \nonumber
\end{align}
The first term can be bounded similarly to Equation~\ref{eq:per-item regret}:
\begin{align}
  & \sum_{e \in \bar{A}^\ast} \sum_{k = 1}^{K_{e, \eps}} \Delta_{e, k} Z_{e, k}(n) \nonumber \\
  & \quad \leq \sum_{e \in \bar{A}^\ast} \frac{16}{\Delta_{e, K_{e, \eps}}} \log n +
  \sum_{e \in \bar{A}^\ast} \sum_{k = 1}^{K_{e, \eps}} \Delta_{e, k} \frac{4}{3} \pi^2 \nonumber \\
  & \quad \leq \frac{16}{\eps} L \log n + \frac{4}{3} \pi^2 K L.
\end{align}
The second term is bounded trivially as:
\begin{align}
  \sum_{e \in \bar{A}^\ast} \sum_{k = K_{e, \eps} + 1}^{K_e} \Delta_{e, k} Z_{e, k}(n) \leq
  \eps K n
\end{align}
because all gaps $\Delta_{e, k}$ are bounded by $\eps$ and the maximum number of suboptimally chosen items in $n$ episodes is $K n$ (Theorem~\ref{thm:regret decomposition}). Based on our upper bounds, we get:
\begin{align}
  R(n) \leq \frac{16}{\eps} L \log n + \eps K n + \frac{4}{3} \pi^2 K L
\end{align}
and then set $\displaystyle \eps = 4 \sqrt{\frac{L \log n}{K n}}$. This concludes our proof.
\end{proof}

\subsection{Lower Bounds}
\label{sec:lower bounds}

We derive an asymptotic lower bound on the expected cumulative regret $R(n)$ that has the same dependence on the gap and $n$ as the upper bound in Theorem~\ref{thm:gap-dependent}. This bound is proved on a class of matroid bandits that are equivalent to $K$ Bernoulli bandits.

Specifically, we prove the lower bound on a \emph{partition matroid bandit}, which is defined as follows. Let $E$ be a set of $L$ items and $B_1, \dots, B_K$ be a partition of this set. Let the family of independent sets be defined as:
\begin{align}
  \cI = \set{I \subseteq E: \left(\forall k: |I \cap B_k| \leq 1\right)}.
  \label{eq:partition independent set}
\end{align}
Then $M = (E, \cI)$ is a \emph{partition matroid} of rank $K$. Let $P$ be a probability distribution over the weights of the items, where the weight of each item is distributed independently of the other items. Let the weight of item $e$ be drawn i.i.d. from a Bernoulli distribution with mean:
\begin{align}
  \bar{\bw}(e) = \left\{
  \begin{array}{ll}
  0.5 & e = \min_{i \in B_k} i \\
  0.5 - \Delta & \textrm{otherwise},
  \end{array}
  \right.
  \label{eq:partition weight}
\end{align}
where $\Delta > 0$. Then $\tilde{B} = (M, P)$ is our partition matroid bandit. The key property of $\tilde{B}$ is that it is equivalent to $K$ independent Bernoulli bandits, one for each partition. The optimal item in each partition is the item with the smallest index, $\min_{i \in B_k} i$. All gaps are $\Delta$.

To formalize our result, we need to introduce the notion of \emph{consistent algorithms}. We say that the algorithm is consistent if for any matroid bandit, any suboptimal $e \in \bar{A}^\ast$, and any $\alpha > 0$, $\mathbb{E}[T_e(n)] = o(n^\alpha)$, where $T_e(n)$ is the number of times that item $e$ is chosen in $n$ episodes. In the rest of our analysis, we focus only on consistent algorithms. This is without loss of generality. In particular, by definition, an inconsistent algorithm performs poorly on some problems, and therefore extremely well on others. Because of this, it is difficult to prove good problem-dependent lower bounds for inconsistent algorithms. Our main claim is below.

\begin{theorem}
\label{thm:lower bound} For any partition matroid bandit $\tilde{B}$ that is defined in Equations~\ref{eq:partition independent set} and \ref{eq:partition weight}, and parameterized by $L$, $K$, and $0 < \Delta < 0.5$; the regret of any consistent algorithm is bounded from below as:
\begin{align*}
  \liminf_{n \rightarrow \infty} \frac{R(n)}{\log n} \geq \frac{L - K}{4 \Delta}.
\end{align*}
\end{theorem}
\begin{proof}
The theorem is proved as follows:
\begin{align}
  \liminf_{n \rightarrow \infty} \frac{R(n)}{\log n}
  & \geq \sum_{k = 1}^K \sum_{e \in B_k - A^\ast}
  \frac{\Delta}{\mathrm{kl}(0.5 - \Delta, 0.5)} \nonumber \\
  & = \frac{(L - K) \Delta}{\mathrm{kl}(0.5 - \Delta, 0.5)} \nonumber \\
  & \geq \frac{L - K}{4 \Delta},
\end{align}
where $\mathrm{kl}(0.5 - \Delta, 0.5)$ is the KL divergence between two Bernoulli variables with means $0.5 - \Delta$ and 0.5. The first inequality is due to Theorem 2.2 \cite{bubeck12regret}, which is applied separately to each partition $B_k$. The second inequality is due to $\mathrm{kl}(p, q) \leq \frac{(p - q)^2}{q (1 - q)}$, where $p = 0.5 - \Delta$ and $q = 0.5$.
\end{proof}

\subsection{Summary of Theoretical Results}
\label{sec:summary}

We prove two upper bounds on the regret of $\omm$, one gap-dependent and one gap-free. These bounds can be summarized as:
\begin{align}
  \begin{aligned}
    \text{Theorem~\ref{thm:gap-dependent}} & \qquad O(L (1 / \Delta) \log n) \\
    \text{Theorem~\ref{thm:gap-free}} & \qquad O(\sqrt{K L n \log n}),
  \end{aligned}
\end{align}
where $\Delta = \min\limits_e \min\limits_{k \in \cO_e} \Delta_{e, k}$. Both bounds are sublinear in the number of episodes $n$, and at most linear in the rank $K$ of the matroid and the number of items $L$. In other words, they scale favorably with all quantities of interest and as a result we expect them to be practical.

Our upper bounds are reasonably tight. More specifically, the gap-dependent upper bound in Theorem~\ref{thm:gap-dependent} matches the lower bound in Theorem~\ref{thm:lower bound}, which is proved on a partition matroid bandit. Furthermore, the gap-free upper bound in Theorem~\ref{thm:gap-free} matches the lower bound of Audibert \etal~\cite{audibert14regret} in adversarial combinatorial semi-bandits, up to a factor of $\sqrt{\log n}$.

Our gap-dependent upper bound has the same form as the bound of Auer \etal~\cite{auer02finitetime} for multi-armed bandits. This observation suggests that the sample complexity of learning a maximum-weight basis of a matroid is comparable to that of the multi-armed bandit. The only major difference is in the definitions of the gaps. We conclude that learning with matroids is extremely sample efficient.

%% file: Experiments.tex
%!TEX root = Paper.tex

\section{Experiments}
\label{sec:experiments}

\begin{figure*}[t]
  \centering
  \includegraphics[width=6.8in, bb=0in 4.25in 8.5in 6.75in]{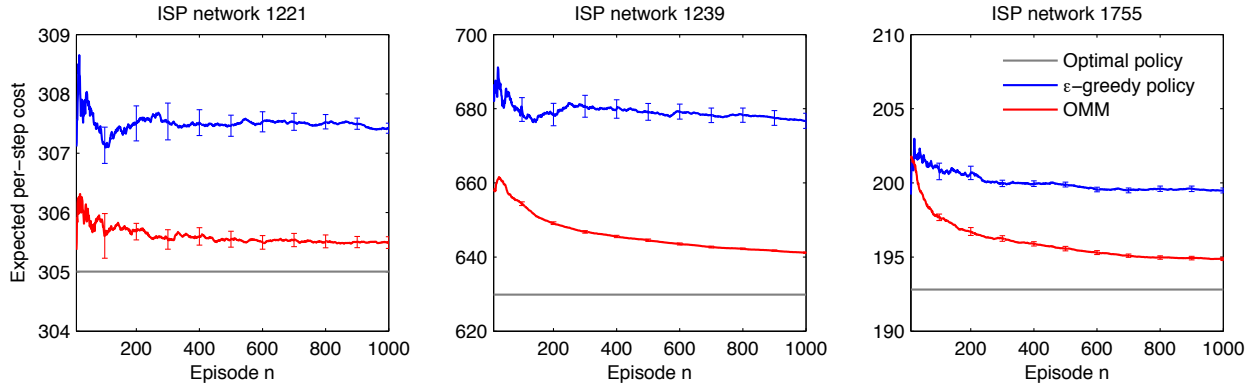}
  \caption{The expected per-step cost of building three minimum spanning trees in up to $10^3$ episodes.}
  \label{fig:latency trends}
\end{figure*}

\begin{table*}
  \centering
  \begin{tabular}{|l|cc|ccc|rrr|} \hline
    ISP & Number & Number & Minimum & Maximum & Average & Optimal & $\eps$-greedy & \\
    network & of nodes & of edges & latency & latency & latency & policy & policy & $\omm$ \\ \hline
    1221 & 108 & 153 & 1 & 17 & 2.78 & 305.00 & $307.42 \pm 0.08$ & $305.49 \pm 0.10$ \\
    1239 & 315 & 972 & 1 & 64 & 3.20 & 629.88 & $676.74 \pm 2.03$ & $641.17 \pm 0.18$ \\
    1755 & 87 & 161 & 1 & 31 & 2.91 & 192.81 & $199.49 \pm 0.16$ & $194.88 \pm 0.11$ \\
    3257 & 161 & 328 & 1 & 47 & 4.30 & 550.85 & $570.35 \pm 0.63$ & $559.80 \pm 0.10$ \\
    3967 & 79 & 147 & 1 & 44 & 5.19 & 306.80 & $320.30 \pm 0.52$ & $308.54 \pm 0.08$ \\
    6461 & 141 & 374 & 1 & 45 & 6.32 & 376.27 & $424.78 \pm 1.54$ & $381.48 \pm 0.07$ \\ \hline
  \end{tabular}
  \caption{The description of six ISP networks from our experiments and the expected per-step cost of building minimum spanning trees on these networks in $10^3$ episodes. All latencies and costs are in milliseconds.}
  \label{tab:latency summary}
\end{table*}

Our algorithm is evaluated on three matroid bandit problems: graphic (Section~\ref{sec:graphic matroid}), transversal (Section~\ref{sec:transversal matroid}), and linear (Section~\ref{sec:linear matroid}).

All experiments are episodic. In each episode, $\omm$ selects a basis $A^t$, observes the weights of the individual items in that basis, and then updates its model of the environment. The performance of $\omm$ is measured by the \emph{expected per-step return} in $n$ episodes:
\begin{align}
  \frac{1}{n} \E{\sum_{t = 1}^n f(A^t, \bw_t)}{\bw_1, \dots, \bw_n},
  \label{eq:per-step regret}
\end{align}
the expected cumulative return in $n$ episodes divided by $n$. $\omm$ is compared to two baselines. The first baseline is the maximum-weight basis $A^\ast$ in expectation. The basis $A^\ast$ is computed as in Equation~\ref{eq:optimal arm} and is our notion of optimality. The second baseline is an $\eps$-greedy policy, where $\eps = 0.1$. This setting of $\eps$ is common in practice and corresponds to 10\% exploration.

%% file: Experiment1.tex
%!TEX root = Paper.tex

\subsection{Graphic Matroid}
\label{sec:graphic matroid}

In the first experiment, we evaluate $\omm$ on the problem of learning a routing network for an Internet service provider (ISP). We make the assumption that the routing network is a spanning tree. Our objective is to learn a tree that has the lowest expected latency on its edges.

Our problem can be formulated as a \emph{graphic matroid bandit}. The ground set $E$ are the edges of a graph, which represents the topology of a network. We experiment with six networks from the \emph{RocketFuel} dataset \cite{spring04measuring}, which contain up to 300 nodes and $10^3$ edges (Table~\ref{tab:latency summary}). A set of edges is considered \emph{independent} if it does not contain a cycle. The latency of edge $e$ is $\bw(e) = \bar{\bw}(e) - 1 + \eps$, where $\bar{\bw}(e)$ is the expected latency, which is recorded in our dataset; and $\eps \sim \mathrm{Exp}(1)$ is exponential noise. The latency $\bar{\bw}(e)$ ranges from one to 64 milliseconds. Our noise model is motivated by the following observation. The latency in ISP networks can be mostly explained by geographical distances \cite{choi04analysis}, the expected latency $\bar{\bw}(e)$. The noise tends to be small, on the order of a few hundred microseconds, and it is unlikely to cause high latency.

In Figure~\ref{fig:latency trends}, we report our results from three ISP networks. We observe the same trends on all networks. First, the expected cost of $\omm$ approaches that of the optimal solution $A^\ast$ as the number of episodes increases. Second, $\omm$ outperforms the $\eps$-greedy policy in less than 10 episodes. The expected costs of all policies on all networks are reported in Table~\ref{tab:latency summary}. We observe again that $\omm$ consistently outperforms the $\eps$-greedy policy, often by a large margin.

$\omm$ learns quickly because all of our networks are sparse. In particular, the number of edges in each network is never more than four times larger than the number of edges in its spanning tree. Therefore, at least in theory, each edge can be observed at least once in four episodes and our method can learn quickly the mean latency of each edge.

%% file: Experiment2.tex
%!TEX root = Paper.tex

\subsection{Transversal Matroid}
\label{sec:transversal matroid}

\begin{figure*}[t]
  \centering
  \raisebox{-0.63in}{\includegraphics[scale=0.23]{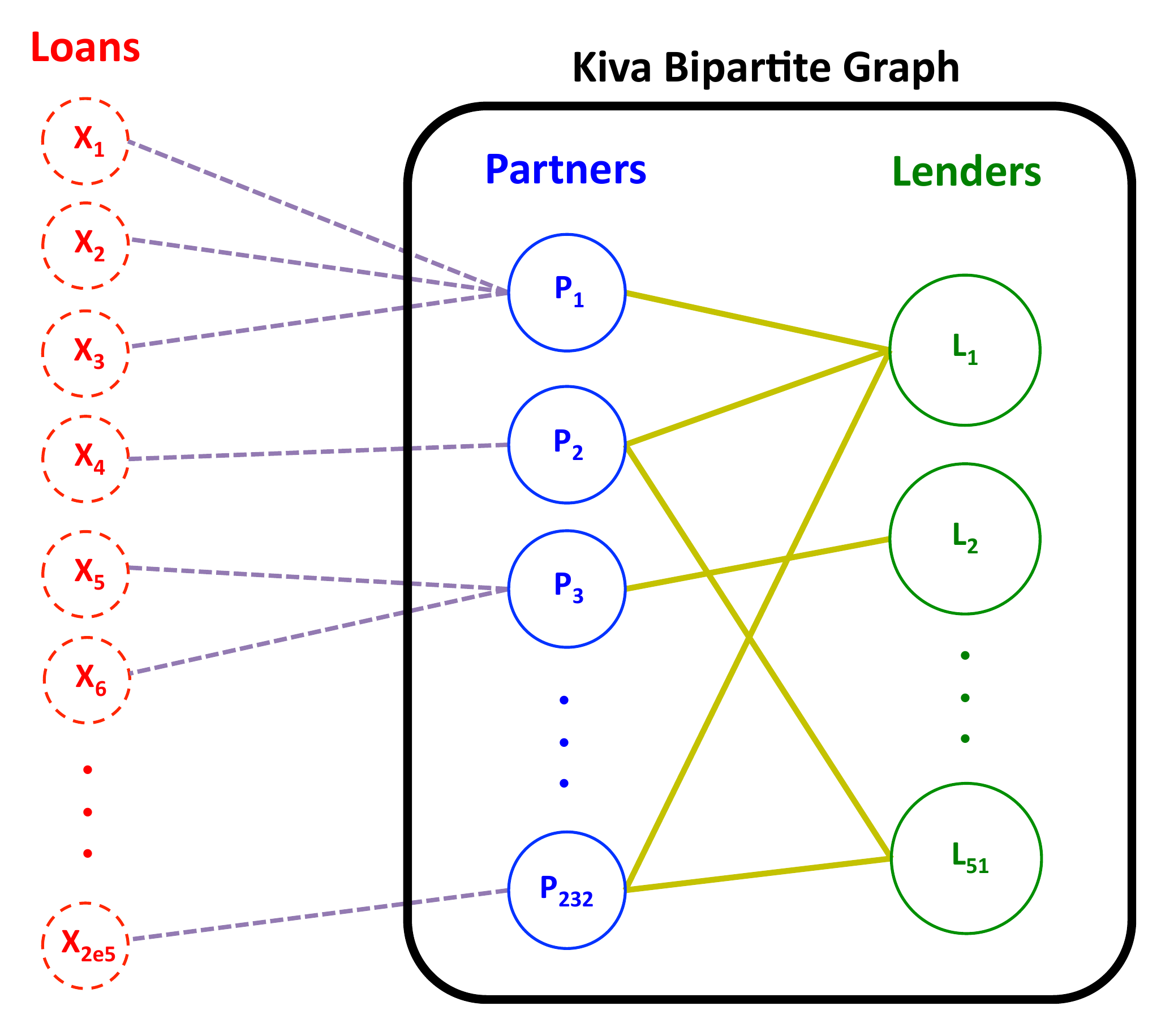}}
  \ \
  \raisebox{-0.85in}{\includegraphics[width=2.4in, bb=2.75in 4.25in 5.75in 6.75in]{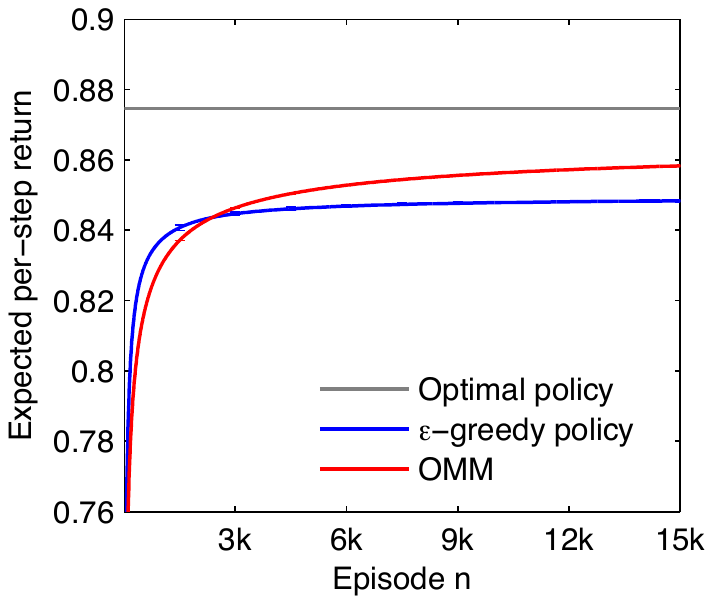}}
  \raisebox{0in}{{\small
  \setlength{\tabcolsep}{4pt}
  \begin{tabular}{|cr|ccr|} \hline
    Partner & $\bar{\bw}(e)$ & Lender & Num of & Avg \\
    id & & id & partners & rate \\ \hline
    46 & 1.0 & 31 & 200 & 0.728 \\
    70 & 1.0 & 2 & 2 & 0.924 \\
    72 & 1.0 & 20 & 195 & 0.725 \\
    88 & 1.0 & 23 & 207 & 0.724 \\
    168 & 0.983 & 44 & 49 & 0.712 \\
    231 & 0.981 & 48 & 186 & 0.723 \\
    179 & 0.970 & 24 & 149 & 0.743 \\
    157 & 0.951 & 10 & 180 & 0.735 \\
    232 & 0.940 & 40 & 10 & 0.718 \\
    123 & 0.934 & 42 & 113 & 0.721 \\ 
    142 & 0.925 & 7 & 168 & 0.745 \\
    130 & 0.919 & 32 & 23 & 0.690 \\ \hline
  \end{tabular}
  }}
  \\
  \vspace{0.1in}
  \hspace{0in} \textbf{(a)} \hspace{1.85in} \textbf{(b)} \hspace{2.05in} \textbf{(c)}
  \caption{\textbf{(a)} The Kiva dataset can be modeled as a bipartite graph connecting lenders to field partners, which, in turn, fund several loans in the region. \textbf{(b)} The expected per-step return of finding maximum weight transversal in up to $15$k episodes. \textbf{(c)} Top 12 selected partners assigned based on their mean success rate in the optimal solution $A^\ast$. The optimal solution involves 46 partner/lender assignments.}
  \label{fig:matching summary}
\end{figure*}

%Each partner is associated with a success rate which represents the probability that a loan handled by this partner is going to be paid back. The objective of the assignment is to maximize the overall success rate of the selected partners. such that the overall success rate of partners is maximized. can be cast as learning a transversal set in a bipartite graph such that
In the second experiment, we study the assignment of lending institutions (known as \emph{partners}) to \emph{lenders} in a microfinance setting, such as Kiva~\cite{kiva}. This problem can be formulated under a family of matroids, called \textit{transversal} matroids~\cite{edmonds65transversals}. The ground set $E$ of a transversal matroid is the set of left vertices of the corresponding bipartite graph, and the independence set $\cI$ consists of the sets of left vertices that belong to all possible matchings in the graph such that no two edges in a matching share an endpoint. The weight $\bar{\bw}(e)$ is the weight associated with the left vertices of the bipartite graph. The goal is to learn a transversal of the bipartite graph that maximizes the overall weight of selected left vertices. 

%Our goal is to maximize the overall success rate, defined here as the expected number of loans that are paid back.  
We used a sample of $194,876$ loans from the Kiva microfinance dataset~\cite{kiva}, and created a bipartite graph. Every loan is handled by a partner (Figure~\ref{fig:matching summary}-a). There are a total of $232$ partners in the dataset that represent the left vertices of the bipartite graph and therefore the ground set $E$ of the matroid. There are $286,874$ lenders in the dataset. We grouped these lenders into $51$ clusters according to the their location: $50$ representing each individual state in United States, and $1$ representing all foreign lenders. These $51$ lender clusters constitute the right vertices of the bipartite graph. There is an edge between a partner and a lender if the lender is among the top $50$\% supporters of the partner, resulting in approximately $5$k edges in the bipartite graph. The weight $\bar{\bw}(e)$ is the probability that a loan handled by partner $e$ will be paid back. We estimate it from the dataset as $\bar{\bw}(e) = \frac{1}{n_l} \sum_{i = 1}^{n_l} \bw_i(e)$, where $n_l$ is the number of loans handled by this partner. We assume $\bw_i(e)$ is $0.7$ if the loan $i$ is in repayment, $1$ if it is paid, and $0$ otherwise. At the beginning of each episode, we choose the loan $i$ at random. 

The optimal solution $A^\ast$ is a transversal in the graph that maximizes the overall success rate of the selected partners. Top twelve partners selected based on their mean success rate in the optimal solution are shown in Figure~\ref{fig:matching summary}-c. For each partner, the id of the lender to which this partner was assigned along with the number of eligible partners of the lender and their average success rate are listed in the Table. The objective of $\omm$ and $\eps$-greedy policies is similar to the optimal policy with the difference that success rates (i.e. $\bw(e)$) are not known beforehand, and they must be learned by interacting repeatedly with the environment. Comparison results of the three policies are reported in Figure~\ref{fig:matching summary}-b. Similar to the previous experiment, we observe the following trends. First, the expected return of $\omm$ approaches that of the optimal solution $A^\ast$ as the number of episodes increases. Second, $\omm$ outperforms the $\eps$-greedy policy.

%% file: Experiment3.tex
%!TEX root = Paper.tex

\subsection{Linear Matroid}
\label{sec:linear matroid}

\begin{figure*}[t]
  \centering
  {\small
  \begin{tabular}{|l|r|l|} \hline
    Movie title & $\bar{\bw}(e)$ & Movie genres \\ \hline
    American Beauty & 0.568 & Comedy Drama \\
    Jurassic Park & 0.442 & Action Adventure Sci-Fi \\
    Saving Private Ryan & 0.439 & Action Drama War \\
    Matrix & 0.429 & Action Sci-Fi Thriller \\
    Back to the Future & 0.428 & Comedy Sci-Fi \\
    Silence of the Lambs & 0.427 & Drama Thriller \\
    Men in Black & 0.420 & Action Adventure Comedy Sci-Fi \\
    Fargo & 0.416 & Crime Drama Thriller \\
    Shakespeare in Love & 0.392 & Comedy Romance \\
    L.A. Confidential & 0.379 & Crime Film-Noir Mystery Thriller \\
    E.T. the Extra-Terrestrial & 0.376 & Children's Drama Fantasy Sci-Fi \\
    Ghostbusters & 0.361 & Comedy Horror \\ \hline
  \end{tabular}
  }
  \ \
  \raisebox{-0.92in}{\includegraphics[width=2.8in, bb=2.5in 4.25in 6in 6.75in]{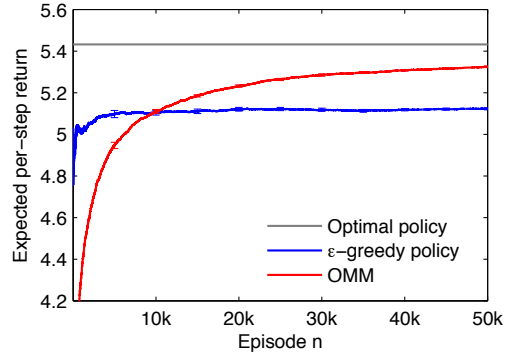}}
  \caption{\textbf{Left}. Twelve most popular movies in the optimal solution $A^\ast$. The optimal solution involves 17 movies. \textbf{Right}. The expected per-step return of three movie recommendation policies in up to $50$k episodes.}
  \label{fig:preferences}
\end{figure*}

In the last experiment, we evaluate $\omm$ on the problem of learning a set of diverse and popular movies. This kind of movies is typically recommended by existing content recommender systems. The movies are popular, and therefore the user is likely to choose them. The movies are diverse, and therefore cover many potential interests of the user.

Our problem can be formulated as a \emph{linear matroid bandit}. The ground set $E$ are movies from the \emph{MovieLens} dataset \cite{movielens}, a dataset of 6 thousand people who rated one million movies. We restrict our attention to 25 most rated movies and 75 movies that are not well known. So the cardinality of $E$ is 100. For each movie $e$, we define a feature vector $\bu_e \in \set{0, 1}^{18}$, where $\bu_e(j)$ indicates that movie $e$ belongs to genre $j$. A set of movies $X$ is considered \emph{independent} if for any movie $e \in X$, the vector $\bu_e$ cannot be written as a linear combination of the feature vectors of the remaining movies in $X$. This is our notion of diversity. The expected weight $\bar{\bw}(e)$ is the probability that movie $e$ is chosen. We estimate it as $\bar{\bw}(e) = \frac{1}{n_p} \sum_{i = 1}^{n_p} \bw_i(e)$, where $\bw_i(e)$ is the indicator that person $i$ rated movie $e$ and $n_p$ is the number of people in our dataset. At the beginning of each episode, we choose the person $i$ at random.

Twelve most popular movies from the optimal solution $A^\ast$ are listed in Figure~\ref{fig:preferences}. These movies cover a wide range of movie genres and appear to be diverse. This validates our assumption that linear independence is suitable for modeling diversity. The expected return of $\omm$ is reported in the same figure. We observe the same trends as in the previous experiments. More specifically, the expected return of $\omm$ approaches that of $A^\ast$ as the number of episodes increases and $\omm$ outperforms the $\eps$-greedy policy in $10$k episodes.

%% file: RelatedWork.tex
%!TEX root = Paper.tex

\section{Related Work}
\label{sec:related work}

Our problem can be viewed as a stochastic combinatorial semi-bandit \cite{gai12combinatorial}, where all feasible solutions are the independent sets of a matroid. Stochastic combinatorial semi-bandits were pioneered by Gai \etal~\cite{gai12combinatorial}, who proposed a UCB algorithm for solving these problems. Chen \etal~\cite{chen13combinatorial} proved that the expected cumulative regret of this method is $O(K^2 L (1 / \Delta) \log n)$. Our gap-dependent regret bound is $O(L (1 / \Delta) \log n)$, a factor of $K^2$ tighter than the bound of Chen \etal~\cite{chen13combinatorial}. Our analysis relies heavily on the properties of our problem and therefore we can derive a much tighter bound.

COMBAND \cite{cesabianchi12combinatorial}, OSMD \cite{audibert14regret}, and FPL \cite{neu13efficient} are algorithms for adversarial combinatorial semi-bandits. The main limitation of COMBAND and OSMD is that they are not guaranteed to be computationally efficient. More specifically, COMBAND needs to sample from a distribution over exponentially many solutions and OSMD needs to project to the convex hull of these solutions. FPL is computationally efficient but not very practical because its time complexity increases with time. On the other hand, $\omm$ is guaranteed to be computationally efficient but can only solve a special class of combinatorial bandits, matroid bandits.

Matroids are a broad and important class of combinatorial optimization problems \cite{oxley11matroid}, which has been an active area of research for the past 80 years. This is the first paper that studies a well-known matroid problem in the bandit setting and proposes a learning algorithm for solving it.

Our work is also related to submodularity \cite{nemhauser78approximation}. In particular, let:
\begin{align}
  g(X) = \max_{Y: Y \subseteq X, Y \in \cI} f(Y, \bar{\bw})
\end{align}
be the maximum weight of an independent set in $X$. Then it is easy to show that $g(X)$ is submodular and monotonic in $X$, and that the maximum-weight basis of a matroid is a solution to $A^\ast = \arg\max_{A: \abs{A} = K} g(A)$. Many algorithms for learning how to maximize a submodular function have been proposed recently \cite{guillory11online,yue11linear,gabillon13adaptive,wen13sequential,gabillon14largescale}. None of these algorithms are suitable for solving our problem. There are two reasons. First, each algorithm is designed to maximize a specific submodular function and our function $g$ may not be of that type. Second, the algorithms are only near optimal, learn a set $A$ such that $g(A) \geq (1 - 1 / e) g(A^\ast)$. Note that our method is guaranteed to be optimal and learn $A^\ast$.

%% file: Conclusions.tex
%!TEX root = Paper.tex

\section{Conclusions}
\label{sec:conclusions}

This is the first work that studies the problem of learning a maximum-weight basis of a matroid, where the weights of the items are initially unknown, and have to be learned by interacting repeatedly with the environment. We propose a practical algorithm for solving this problem and bound its regret. The regret is sublinear in time and at most linear in all other quantities of interest. We evaluate our method on three real-world problems and show that it is practical.

Our regret bounds are $\Omega(\sqrt{L})$. Therefore, $\omm$ is not practical when the number of items $L$ is large. We believe that these kinds of problems can be solved efficiently by introducing additional structure, such as \emph{linear generalization}. In this case, the weight of each item would be modeled as a linear function of its features and the goal is to learn the parameters of this function.

Many combinatorial optimization problems can be viewed as optimization on a matroid or its generalizations, such as \emph{maximum-weight matching} on a bipartite graph and \emph{minimum cost flows}. In a sense, these are the hardest problems in combinatorial optimization that can be solved optimally in polynomial time \cite{papadimitriou98combinatorial}. In this work, we show that one of these problems is efficiently learnable. We believe that the key ideas in our work are quite general and can be applied to other problems that involve matroids.

%% file: Appendix.tex
%!TEX root = Paper.tex

\clearpage
\onecolumn
\appendix

\section{Technical Lemmas}
\label{sec:lemmas}

\newtheorem*{lem:bijection}{Lemma~\ref{lem:bijection}}
\begin{lem:bijection}
For any two matroid bases $A^\ast$ and $A^t$, there exists a bijection $\pi: \set{1, \dots, K} \to \set{1, \dots, K}$ such that:
\begin{align*}
  \set{a^t_1, \dots, a^t_{k - 1}, a^\ast_{\pi(k)}} \in \cI
  \quad \forall k = 1, \dots, K.
\end{align*}
In addition, $\pi(k) = i$ when $a^t_k = a^\ast_i$ for some $i$.
\end{lem:bijection}
\begin{proof}
Our proof is constructive. The key idea is to exchange items in $A^t$ for items in $A^\ast$ in backward order, from $a^t_K$ to $a^t_1$. For simplicity of exposition, we first assume that $A^\ast \cap A^t = \emptyset$.

First, we exchange item $a^t_K$. In particular, from the augmentation property of a matroid, we know that there exists an item $a^\ast_i \in A^\ast - (A^t - a^t_K)$ such that $A^t - a^t_K + a^\ast_i \in \cI$. We choose any such item $a^\ast_i$ and exchange it for $a^t_K$. The result is a basis:
\begin{align}
  B_{K - 1} = \set{a^t_1, \dots, a^t_{K - 1}, a^\ast_{\pi(K)}} \in \cI,
\end{align}
where $\pi(K) = i$. Second, we apply the same idea to item $a^t_{K - 1}$. In particular, from the augmentation property, we know that there exists an item $a^\ast_i \in A^\ast - (B_{K - 1} - a^t_{K - 1})$ such that $B_{K - 1} - a^t_{K - 1} + a^\ast_i \in \cI$. We select any such item $a^\ast_i$ and exchange it for $a^t_{K - 1}$. The result is another basis:
\begin{align}
  B_{K - 2} = \set{a^t_1, \dots, a^t_{K - 2}, a^\ast_{\pi(K - 1)}, a^\ast_{\pi(K)}} \in \cI,
\end{align}
where $\pi(K - 1) = i$. The same argument applies to item $a^t_{K - 2}$, all the way down to item $a^t_1$. The result is a sequence of bases:
\begin{align}
  B_{k - 1} = \set{a^t_1, \dots, a^t_{k - 1}, a^\ast_{\pi(k)}, \dots, a^\ast_{\pi(K)}} \in \cI
  \quad \forall k = 1, \dots, K.
\end{align}
Our main claim follows from the hereditary property of a matroid, any subset of an independent set is independent.

Finally, suppose that $A^\ast \cap A^t \neq \emptyset$. Then our construction changes in only one step. In any step $k$, we set $\pi(k)$ to $i$ when $a^t_k = a^\ast_i$ for some $i$. The items $a^t_k$ and $a^\ast_i$ can be always exchanged because $a^\ast_i \notin B_k - a^t_k$. Otherwise, $B_k$ would be a set with two identical items, $a^t_k$ and $a^\ast_i$, which contradicts to the fact that $B_k$ is a basis.
\end{proof}

\begin{lemma}
\label{lem:pulls} For any item $e \in \bar{A}^\ast$ and $k \leq K_e$:
\begin{align*}
  \E{\sum_{t = 1}^n \mathds{1}_{e, k}(t) \I{T_e(t - 1) > \ell}}{\bw_1, \dots, \bw_n} \leq
  \frac{4}{3} \pi^2
\end{align*}
when $\ell = \floors{\frac{8}{\Delta_{e, k}^2} \log n}$.
\end{lemma}
\begin{proof}
First, note that the event $\mathds{1}_{e, k}(t)$ implies $U_t(e) \geq U_t(a^\ast_k)$ (Theorem~\ref{thm:regret decomposition}). Second, by the design of $\omm$, the counter $T_e(t)$ increases when the event $\mathds{1}_{e, k}(t)$ happens, for any $k$. Based on these facts, it follows that:
\begin{align}
  \sum_{t = 1}^n \mathds{1}_{e, k}(t) \I{T_e(t - 1) > \ell}
  & = \sum_{t = \ell + 1}^n \mathds{1}_{e, k}(t) \I{T_e(t - 1) > \ell} \nonumber \\
  & \leq \sum_{t = \ell + 1}^n \I{U_t(e) \geq U_t(a^\ast_k), \ T_e(t - 1) > \ell} \nonumber \\
  & \leq \sum_{t = \ell + 1}^n \sum_{s = 1}^t \sum_{s_e = \ell + 1}^t
  \I{\hat{w}_{e, s_e} + c_{t - 1, s_e} \geq \hat{w}_{a^\ast_k, s} + c_{t - 1, s}} \nonumber \\
  & = \sum_{t = \ell}^{n - 1} \sum_{s = 1}^{t + 1} \sum_{s_e = \ell + 1}^{t + 1}
  \I{\hat{w}_{e, s_e} + c_{t, s_e} \geq \hat{w}_{a^\ast_k, s} + c_{t, s}}.
\end{align}
When $\hat{w}_{e, s_e} + c_{t, s_e} \geq \hat{w}_{a^\ast_k, s} + c_{t, s}$, at least one of the following events must happen:
\begin{align}
  \hat{w}_{a^\ast_k, s} & \leq \bar{\bw}(a^\ast_k) - c_{t, s} \label{eq:H1} \\
  \hat{w}_{e, s_e} & \geq \bar{\bw}(e) + c_{t, s_e} \label{eq:H2} \\
  \bar{\bw}(a^\ast_k) & < \bar{\bw}(e) + 2 c_{t, s_e}. \label{eq:minimum pulls}
\end{align}
We bound the probability of the first two events (Equations~\ref{eq:H1} and \ref{eq:H2}) using Hoeffding's inequality:
\begin{align}
  P(\hat{w}_{a^\ast_k, s} \leq \bar{\bw}(a^\ast_k) - c_{t, s})
  & \leq \exp[-4 \log t] = t^{-4} \\
  P(\hat{w}_{e, s_e} \geq \bar{\bw}(e) + c_{t, s_e})
  & \leq \exp[-4 \log t] = t^{-4}.
\end{align}
When $s_e \geq \frac{8}{\Delta_{e, k}^2} \log n$, the third event (Equation~\ref{eq:minimum pulls}) cannot happen because:
\begin{align}
  \bar{\bw}(a^\ast_k) - \bar{\bw}(e) - 2 c_{t, s_e} =
  \Delta_{e, k} - 2 \sqrt{\frac{2 \log t}{s_e}} \geq
  0.
\end{align}
This is guaranteed when $\ell = \floors{\frac{8}{\Delta_{e, k}^2} \log n}$. Finally, we combine all of our claims and get:
\begin{align}
  \E{\sum_{t = 1}^n \mathds{1}_{e, k}(t) \I{T_e(t - 1) > \ell}}{\bw_1, \dots, \bw_n}
  & \leq \sum_{t = \ell}^{n - 1} \sum_{s = 1}^{t + 1} \sum_{s_e = \ell + 1}^{t + 1}
  \big[P(\hat{w}_{a^\ast_k, s} \leq \bar{\bw}(a^\ast_k) - c_{t, s}) + {} \nonumber \\
  & \hspace{1.1in} P(\hat{w}_{e, s_e} \geq \bar{\bw}(e) + c_{t, s_e})\big] \nonumber \\
  & \leq \sum_{t = 1}^\infty 2 (t + 1)^2 t^{-4} \nonumber \\
  & \leq \sum_{t = 1}^\infty 8 t^{-2} \nonumber \\
  & =\frac{4}{3} \pi^2.
\end{align}
The last step is due to the fact that $\displaystyle \sum_{t = 1}^\infty t^{-2} = \frac{\pi^2}{6}$.
\end{proof}

\begin{lemma}
\label{lem:multiple optimal pulls} Let $\Delta_1 \geq \ldots \geq \Delta_K$ be a sequence of $K$ positive numbers. Then:
\begin{align*}
  \left[\Delta_1 \frac{1}{\Delta_1^2} + \sum_{k = 2}^K \Delta_k
  \left(\frac{1}{\Delta_k^2} - \frac{1}{\Delta_{k - 1}^2}\right)\right] \leq
  \frac{2}{\Delta_K}.
\end{align*}
\end{lemma}
\begin{proof}
First, we note that:
\begin{align}
  \left[\Delta_1 \frac{1}{\Delta_1^2} + \sum_{k = 2}^K \Delta_k
  \left(\frac{1}{\Delta_k^2} - \frac{1}{\Delta_{k - 1}^2}\right)\right] =
  \sum_{k = 1}^{K - 1} \frac{\Delta_k - \Delta_{k + 1}}{\Delta_k^2} + \frac{1}{\Delta_K}.
\end{align}
Second, by our assumption, $\Delta_k \geq \Delta_{k + 1}$ for all $k < K$. Therefore:
\begin{align}
  \sum_{k = 1}^{K - 1} \frac{\Delta_k - \Delta_{k + 1}}{\Delta_k^2} + \frac{1}{\Delta_K}
  & \leq \sum_{k = 1}^{K - 1} \frac{\Delta_k - \Delta_{k + 1}}{\Delta_k \Delta_{k + 1}} +
  \frac{1}{\Delta_K} \nonumber \\
  & = \sum_{k = 1}^{K - 1} \left[\frac{1}{\Delta_{k + 1}} - \frac{1}{\Delta_k}\right] +
  \frac{1}{\Delta_K} \nonumber \\
  & = \frac{2}{\Delta_K} - \frac{1}{\Delta_1} \nonumber \\
  & < \frac{2}{\Delta_K}.
\end{align}
This concludes our proof.
\end{proof}